\documentclass[runningheads]{llncs}
\usepackage[T1]{fontenc}
\usepackage{graphicx}
\usepackage{acronym}
\usepackage{bm}
\usepackage{subfig}
\usepackage{paralist}
\usepackage{booktabs} 
\usepackage{url}      
\usepackage{amsmath}  
\usepackage{amssymb}
\usepackage[hidelinks]{hyperref}
\usepackage[nameinlink,capitalise]{cleveref}
\usepackage{bbm}
\usepackage{color}

\acrodef{RV}[RV]{random variable}
\acrodef{VE}[VE]{variable elimination}
\acrodef{TN}[TN]{tensor network}
\acrodef{BTN}[BTN]{base tensor network}
\acrodef{BT}[BT]{base tensor}
\acrodef{CT}[CT]{core tensor}
\acrodef{LT}[LT]{linked tensor}
\acrodef{PGM}[PGM]{probabilistic graphical model}
\acrodef{PS}[PS]{parametric structure}
\acrodef{PTD}[PTD]{parametric tensor decomposition}
\acrodef{FG}[FG]{factor graph}

\begin{document}
\title{On Probabilistic Inference Through Parametric Tensor Decomposition in Base Tensor Networks}
%
%
\author{Sagad Hamid\orcidID{0009-0005-1286-904X} \and
Tanya Braun\orcidID{0000-0003-0282-4284}}
\authorrunning{S. Hamid and T. Braun}
%
\institute{Computer Science Department, University of Münster, Münster, Germany
\email{\{sagad.hamid,tanya.braun\}@uni-muenster.de}}
\maketitle              
\begin{abstract}
Probabilistic inference is generally only tractable in low-treewidth graphical models, limiting its effective applicability in high-treewidth settings. Many existing methods improve efficiency by exploiting specific parametric structure, such as symmetries. However, they typically require such structure to be explicitly present, limiting their applicability to a broader range of graphical models.
To address this limitation, we propose a framework where tractable inference is controlled by latent parametric structure exploitation, rather than requiring it to be explicitly present a priori. Our approach first reparameterises a graphical model as a specific tensor network representation, which we call a \emph{base tensor network}. This representation yields two key properties that allow inference tractability to be controlled by parametric structure: 1) First, the complexity of inference is mainly determined by the parametric structure of a single tensor, called the \emph{base tensor}. We characterise several tractable classes of base tensors for which the entire base tensor network can be contracted efficiently. 2) Second, decomposing the base tensor yields again a collection of base tensor networks. This allows inference to be naturally reduced to decomposing the base tensor into tractable components with sufficient parametric structure. We call this procedure \emph{parametric tensor decomposition}. By exploiting parametric structure within the base tensor, our framework enables a novel view on inference beyond settings where such structure is explicitly present.

\keywords{Probabilistic Graphical Models  \and Probabilistic Inference \and Tensor Networks \and Tensor Network Contraction.}
\end{abstract}
\section{Introduction}

\Acp{PGM} are a widely used framework for probabilistic modeling \cite{koller2009probabilistic,wang2013markov,dellaert2017factor,chater2006probabilistic,darwiche2022causal}. By exploiting (conditional) independencies, they represent a joint distribution over \acp{RV} in factorised form. Classical inference algorithms such as \ac{VE} \cite{zhang1996exploiting} leverage this structure to compute quantities like the partition function more efficiently, but their complexity is exponential in the model’s treewidth \cite{darwiche2001recursive,darwiche2009modeling}.

While low-treewidth models permit tractable inference, inference in general \acp{PGM} is \#P-hard \cite{cooper1990computational,dechter2022reasoning}. To address this, exploiting \ac{PS} that goes beyond independencies has proven to be highly effective. We use \ac{PS} to denote local or global regularities in a model’s parameters that enable more efficient inference, e.g., determinism, context-specific independencies, equal parameters, or symmetries. Such structure often admits more compact representations, including circuit-based models \cite{sidheekh2024building,chavira2008probabilistic,choi2020probabilistic} or lifted representations \cite{kimmig2015lifted,richardson2006markov,braun2016lifted}, which significantly reduce time and space complexity for probabilistic inference.

\paragraph{Motivation}
Many existing approaches exploit \ac{PS} only when it is explicitly encoded in the model, thereby missing opportunities to reveal and leverage latent \ac{PS} in more general models. For example, symmetry exploitation is typically studied in relational models, where symmetries are expected to arise naturally. Similarly, compiling a \ac{PGM} into a circuit-based representation is most effective when specific forms of \ac{PS}, such as context-specific independence or determinism, are already present \cite{chavira2008probabilistic}. In the absence of such explicit structure, these representations and algorithms often yield only limited computational benefits.
Crucially, however, \ac{PS} can still provide a powerful means for reducing inference complexity from exponential to polynomial. This motivates the development of a framework that makes \ac{PS} exploitation the central component of tractable inference. Since directly identifying and exploiting latent \ac{PS} in arbitrary \acp{PGM} is non-trivial, we propose a framework in which a \ac{PGM} is represented in terms of components with sufficient \ac{PS} where each component enables efficient inference.

\paragraph{Contributions}
In this work, we propose a novel framework for \ac{PS} exploitation in more general \acp{PGM}. Our approach first reparameterises a \ac{PGM} as a specific \ac{TN} representation, recasting probabilistic inference as \ac{TN} contraction. This transformation yields a new class of representations, which we call \emph{\acp{BTN}}.
We show that \acp{BTN} have two key properties that allow inference tractability to be controlled entirely through \ac{PS}. First, their contraction complexity is determined solely by the \ac{PS} of a single tensor, namely the \emph{\ac{BT}}. Based on this property, we characterise several tractable classes of \acp{BT} and show that the entire \ac{BTN} can be contracted efficiently when the corresponding \ac{BT} exhibits sufficient \ac{PS}.
Second, we show that decomposing the \ac{BT} yields a collection of \acp{BTN}. The contraction results of the resulting \acp{BTN} can then be accumulated to obtain the inference result of the original \ac{PGM}. Crucially, this allows inference in \acp{PGM} to be naturally reduced to decomposing the \ac{BT} of a \ac{BTN} into tractable components with sufficient \ac{PS}. We refer to this procedure as \textit{\ac{PTD}}. Finally, we show how to effectively minimise the Frobenius norm to approximate a \ac{BTN} representation with sufficient \ac{PS} that enables efficient inference. By exploiting \ac{PS} within \acp{BTN}, our framework enables a novel view on inference beyond settings where such structure is explicitly present a priori.

\paragraph{Structure} The remainder of this work is structured as follows: \cref{section:background} provides notations and background. \cref{section:base_tensor_networks} introduces \acp{BTN} and shows how to reparameterise a \ac{PGM} into a \ac{BTN}. \cref{section:tractability_parametric_tensor_decomposition} presents tractability results for specific classes of \acp{BT} and further shows how to contract a \ac{BTN} through \ac{PTD}. \cref{section:learning_base_tensor_networks} then shows how \acp{BTN} with sufficient \ac{PS} can be obtained by minimising the Frobenius norm. \cref{section:related_work} then presents related work, and \cref{section:conclusion} concludes our work with potential future work.

\section{Fundamentals}
\label{section:background}
In this section, we specify notations and introduce fundamentals on \acp{FG} as a general \ac{PGM} formalism and inference by \ac{VE}. We then present fundamentals on tensors, \acp{TN}, and \ac{TN} contraction. \Acp{RV} are assumed to be discrete. In this work, we only consider calculating the partition function as inference task, however, we point out that other queries, as calculating marginals or probabilities of events, are reducible to partition function calculations by conditioning \cite{darwiche2001recursive}.

\paragraph{Notation} We define \([d]:=\{0,\dots,d-1\}\) for \(d \in \mathbb{N}\). We denote RVs by uppercase letters \(X\), their domain by \(\text{Dom}(X)\), and their values by lowercase letters \(x \in \text{Dom}(X)\). W.l.o.g.,\ we assume $\text{Dom}(X) = [m]$ for an \ac{RV} with $m$ values. Boldface uppercase letters \(\bm{X}\) and boldface lowercase letters \(\bm{x}\) denote sets of RVs and their values, respectively.  
We denote scalars, vectors, matrices, and tensors by lowercase letters \(s\), lowercase boldface letters \(\bm{v}\), uppercase boldface letters \(\bm{M}\), and calligraphic letters \(\mathcal{T}\), respectively. Sets of tensors are denoted by boldface calligraphic letters \(\bm{\mathcal{T}}\).

\subsection{Factor Graphs and Probabilistic Inference by Variable Elimination}
\paragraph{Factor Graphs.}
A \acf{FG} $G = (\bm{X}, \bm{\Phi})$ is an undirected bipartite graph with \(n=|\bm{X}|\) variable nodes, \(m=|\bm{\Phi}|\) factor nodes, and edges \(\bm{E} \subseteq \bm{X} \times \bm{\Phi}\). Each variable node represents an \ac{RV} \(X \in \bm{X}\) and each factor node a positive function \(\phi_i: \bm{X}_i \rightarrow \mathbb{R}_+\) over \acp{RV} \(\mathrm{rv}(\phi_i)=\bm{X}_i \subseteq \bm{X}\) mapping assignments \(\bm{x}_i \in \mathrm{Dom}(\bm{X}_i)\) to potentials \(\phi_i(\bm{x}_i) \in \mathbb{R}_+\). A factor node is connected to a variable node if the \ac{RV} appears as an argument in the factor. The resulting graph depicts a joint probability distribution \(P_G(\bm{X})\) in terms of the following factorisation
\begin{equation}
\label{equation:FGFactorisation}
P_G(\bm{X}) = \frac{1}{Z(G)}\prod_{i=0}^{m-1}\phi_{i}(\bm{X}_i),
\end{equation}
with \(Z(G)\) called the partition function of \(G\) being defined as
\begin{equation}
\label{equation:PartFunc}
Z(G) = \sum_{X \in \bm{X}}\prod_{i=0}^{m-1}\phi_{i}(\bm{X}_i).
\end{equation}

\paragraph{Probabilistic Inference.}
Instead of operating on the exponentially large distribution \(P(\bm{X})\), probabilistic inference, such as querying marginals \(P_G(\bm{X}_Q), \bm{X}_Q \subseteq \bm{X}\), is performed on the factorised representation using the following two operations, namely multiplication and marginalisation. For \(\bm{X}_i,\bm{X}_j,\bm{X}_{k} \subseteq \bm{X},\) multiplying two factors \(\phi_i(\bm{X}_i)\) and \(\phi_j(\bm{X}_j)\) results in a new factor \(\phi_{k}(\bm{X}_{k}),\) where \(\bm{X}_{k} = \bm{X}_i \cup \bm{X}_j\) and \(\phi_k(\bm{x}_{k}) = \phi_i(\bm{x}_i) \cdot \phi_j(\bm{x}_j)\) with \(\pi_{\bm{X}_i}(\bm{x}_{k}) = \bm{x}_i\) and \(\pi_{\bm{X}_j}(\bm{x}_{k}) = \bm{x}_j\), where \(\pi_{(\bm{\cdot})}\) projects its input to the subscript \acp{RV}. For \(\bm{X}_i' \subseteq \bm{X}_i\), marginalising out \acp{RV} \(\bm{X}_i'\) yields a new factor \(\phi_i'(\bm{X}_i \backslash \bm{X}_i') = \sum_{\bm{X}_i'}\phi_j(\bm{X}_i)\).
The following theorems are commonly used to marginalise out \acp{RV} efficiently \cite{darwiche2009modeling}.
\begin{theorem}
\label[theorem]{theorem:VESumOut1}
For \(\phi_i(\bm{X}_i)\) with \(\bm{X}_i'\),\(\bm{X}_i'' \subseteq \bm{X}_i\) and \(\bm{X}_i' \neq \bm{X}_i'',\) it holds that
\(\sum_{\bm{X}_i'}\sum_{\bm{X}_i''}\phi_i(\bm{X}_i) = \sum_{\bm{X}_i''}\sum_{\bm{X}_i'}\phi_i(\bm{X}_i)\).
\end{theorem}
\begin{theorem}
\label[theorem]{theorem:VESumOut2}
For \(\phi_i(\bm{X}_i),\phi_j(\bm{X}_j)\) with \(\bm{X}'_j \subseteq \bm{X}_j, \bm{X}'_j \cap \bm{X}_i = \emptyset\), it holds that
\(\sum_{\bm{X}'_j} \phi_j(\bm{X}_i) \cdot \phi_k(\bm{X}_j) = \phi_j(\bm{X}_i) \cdot \sum_{\bm{X}'_j} \phi_k(\bm{X}_j)\).
\end{theorem}

\paragraph{Variable Elimination.}
\cref{theorem:VESumOut1} implies that \acp{RV} within a factor can be margin\-alised out in any order, while \cref{theorem:VESumOut2} yields that marginalising out any chosen \ac{RV} \(X\) requires multiplying only factors in which \(X\) appears as an argument. To compute, e.g., \(P_G(\bm{X}_Q)\), \ac{VE} based algorithms multiply factors successively to marginalise out \acp{RV} \(\bm{X}\setminus\bm{X}_Q\) one by one in a given elimination order. During this process, intermediate factors are created with sizes being exponential in the number of \acp{RV} included. As different marginalisation orders may create different intermediate factor sizes, an algorithm's efficiency is significantly improved by choosing an optimal ordering that yields the smallest intermediate factor size. This size is quantified using the graph's induced width, indicating how many \acp{RV} are joined during multiplication. The smallest possible width is called the treewidth and corresponds to an optimal ordering for \ac{VE}.

\subsection{Tensors, Tensor Networks, and Tensor Network Contraction}
\label{subsection:TensorsAndTensorNetworks}

\paragraph{Tensors.}
Let $n \in \mathbb{N}$ and $\bm{d} = (d_0,\dots,d_{n-1}) \in \mathbb{N}^n$. 
For each mode $k \in [n]$, define the index set $I_k := [d_k]$. 
An \emph{$n$-th order tensor} is a mapping
\[
\mathcal{T} : I_0 \times \cdots \times I_{n-1} \to \mathbb{R},
\]
which we identify with $\mathcal{T} \in \mathbb{R}^{\bm{d}}$.
Scalars, vectors, and matrices correspond to tensors of order $0$, $1$, and $2$, respectively. 
Tensors of order $\geq 3$ are called higher-order tensors.
For $\bm{i} = (i_0,\dots,i_{n-1}) \in I_0 \times \cdots \times I_{n-1}$, 
the corresponding entry is denoted by $\mathcal{T}_{\bm{i}}$.

\paragraph{Basic Operations.}
Let $\mathcal{T}^{(0)}, \mathcal{T}^{(1)} \in \mathbb{R}^{\bm{d}}$. 
For $\bullet \in \{+,\odot\}$, define the operation
\[
(\mathcal{T}^{(0)} \bullet \mathcal{T}^{(1)})_{\bm{i}} 
= \mathcal{T}^{(0)}_{\bm{i}} \bullet \mathcal{T}^{(1)}_{\bm{i}},
\]
where $+$ denotes the element-wise sum and $\odot$ denotes the Hadamard product.
For tensors $\mathcal{Q} \in \mathbb{R}^{\bm{d}_q},\mathcal{R} \in \mathbb{R}^{\bm{d}_r},$, 
the outer product $\mathcal{Q} \otimes \mathcal{R}$ is defined by
\[
(\mathcal{Q} \otimes \mathcal{R})_{\bm{i},\bm{j}} 
= \mathcal{Q}_{\bm{i}} \, \mathcal{R}_{\bm{j}},
\]
yielding a tensor of order $|\bm{d}_q| + |\bm{d}_r|$. 
We write $\mathcal{T}^{\otimes n}$ to denote  $\mathcal{T} \otimes \dots \otimes \mathcal{T}$ (\(n\) times).

\paragraph{Contraction.}
For tensors $\mathcal{Q} \in \mathbb{R}^{\bm{d}_q}$ and 
$\mathcal{R} \in \mathbb{R}^{\bm{d}_r}$ the contraction along modes $k$ and $l$ representing the same index of dimension $d_k$ is defined as 
\[
(\mathcal{Q} \times_{k,l} \mathcal{R})_{\hat{\bm{i}},\hat{\bm{j}}}
= \sum_{c=0}^{d_k-1} 
\mathcal{Q}_{i_0,\dots,i_{k-1},c,i_{k+1},\dots} \,
\mathcal{R}_{j_0,\dots,j_{l-1},c,j_{l+1},\dots},
\]
where $\hat{\bm{i}}, \hat{\bm{j}}$ denote the remaining indices.
For order-$2$ tensors, this naturally reduces to standard matrix-matrix multiplication \(\bm{C}_{ij} = \sum_{c=0}^{d}\bm{A}_{ic}\bm{B}_{cj}\).

\paragraph{Tensor networks.}
A \acf{TN} $T = (\bm{\mathcal{T}}, \bm{I})$ is a graph where $\bm{\mathcal{T}}$ is a set of nodes representing tensors and $\bm{I}$ is a set of edges representing indices for the corresponding modes. For a tensor \(\mathcal{T}^{(i)} \in \bm{\mathcal{T}},\) we denote its set of indices by \(\bm{I}(\mathcal{T}^{(i)})\). For two tensors \(\mathcal{T}^{(i)}, \mathcal{T}^{(j)} \in \bm{\mathcal{T}},\) an index \(I \in (\bm{I}(\mathcal{T}^{(i)}) \cap \bm{I}(\mathcal{T}^{(j)}))\) is graphically depicted with an edge connecting these tensors. Generally, each edge in a \ac{TN} either connects two tensors (a \emph{shared index}), or is incident to a single tensor (a \emph{free index}). Consequently, contracting a free index does not require multiplying tensors. 
Contracting all shared edges yields a contraction result which is again a tensor which we denote by $Z(T)$.  The number of remaining free indices after contraction equals the order of $Z(T)$.

A tensor $\mathcal{T} \in \mathbb{R}^{\bm{d}}$ requires $\mathcal{O}(\prod_{i=0}^{n-1} d_i)$ parameters, 
which is exponential in $n$ for fixed $d_i = d$. 
\acp{TN} provide compact representations enabling more efficient computation. Moreover, tensors within a \ac{TN} may exhibit \ac{PS} as well. For \(\bm{d} = (d_0, \dots, d_{n-1})\) we consider the following structures:

\begin{definition}[\textbf{Rank-one Tensor}]
A tensor $\mathcal{T} \in \mathbb{R}^{\bm{d}}$ is rank-one if there exist vectors 
$\bm{v}_k \in \mathbb{R}^{d_k}$ such that
\[
\mathcal{T} = \bm{v}_0 \otimes \cdots \otimes \bm{v}_{n-1}.
\]
\end{definition}

\begin{definition}[\textbf{CP Decomposition}]
A CP decomposition of $\mathcal{T} \in \mathbb{R}^{\bm{d}}$ is
\[
\mathcal{T} = \sum_{i=0}^{r-1} \bm{v}_0^{(i)} \otimes \cdots \otimes \bm{v}_{n-1}^{(i)}.
\]
The minimal such $r$ is called the CP rank of $\mathcal{T}$.
\end{definition}

\begin{definition}[\textbf{Symmetric Tensor}]
A tensor $\mathcal{T} \in \mathbb{R}^{\bm{d}}$ with \(d_0 = \dots = d_{n-1}\) is symmetric if
\[
\mathcal{T}_{i_0,\dots,i_{n-1}} = \mathcal{T}_{i_{\sigma(0)},\dots,i_{\sigma(n-1)}}
\quad \text{for all permutations } \sigma \text{of } [n].
\]
\end{definition}

\section{Base Tensor Networks}
\label{section:base_tensor_networks}

In this section, we introduce \acp{BTN} as the central representation of our framework. Starting from an arbitrary \ac{FG}, we present the construction of \acp{BTN} through reparameterisation using auxiliary \acp{RV} and auxiliary factors. The resulting transformation yields an equivalent representation of the partition function in terms of \ac{TN} contraction. We then present tractable classes of \acp{BTN} and show how latent \ac{PS} can be systematically induced and exploited through \ac{PTD} as the key component for tractable inference.

\subsection{From Factor Graphs to Base Tensor Networks}
\label{subsection:from_factor_graphs_to_base_tensor_networks}

In the following, let \(G=(\bm{X}, \bm{\Phi})\) be an \ac{FG} over discrete \acp{RV} \(\bm{X}\)  with \(n = |\bm{X}|,m=|\bm{\Phi}|\). We first show that every factor can be reparameterised by introducing one auxiliary \ac{RV} for each of its arguments. The auxiliary factors connecting original and auxiliary \acp{RV} are chosen to correspond to invertible matrices. This ensures that the transformation is algebraically reversible in terms of factor operations, namely factor multiplication and marginalisation.

\begin{lemma}[Factor Reparameterisation]
\label[lemma]{lemma:FactorisationThroughARVs}
Let \(\phi_i(\bm{X}_i) \in \bm{\Phi}\) with \(l=|\bm{X}_i|\).
For \(X_{i_j} \in \bm{X}_i\), let \(\bar{X}_{i_j}\) be an auxiliary \ac{RV} with
\(
|\mathrm{Dom}(\bar{X}_{i_j})| = |\mathrm{Dom}(X_{i_j})| = d_{i_j},
\)
and let
\(\bar{\bm{X}}_i = \{\bar{X}_{i_0},\dots,\bar{X}_{i_{l-1}}\}\).
Further, let \(\bar{\phi}_{i_j}(X_{i_j}, \bar{X}_{i_j})\) be auxiliary factors that correspond to invertible matrices. Then there exists a factor \(\bar{\phi}_i(\bar{\bm{X}}_i)\) such that
\begin{align}
\phi_i(\bm{X}_i)
=
\sum_{\bar{\bm{X}}_i}
\bar{\phi}_i(\bar{\bm{X}}_i)
\prod_{j=0}^{l-1}
\bar{\phi}_{i_j}(X_{i_j}, \bar{X}_{i_j}).
\end{align}
\end{lemma}

The proof can be found in the appendix. \cref{lemma:FactorisationThroughARVs} is the key local reparameterisation used to transform an \ac{FG} into a \ac{TN}. Such transformations are closely related to Forney-style \acp{PGM} \cite{forney2001codes}, and have already been applied in the context of, e.g., Gauge transformations \cite{ahn2017gauging,ahn2018gauged}. However, we use it in this work explicitly to propose a novel perspective on inference by introducing a topology that permits effective \ac{PS} exploitation through \ac{PTD}. To do so, we apply it independently to every factor and use for each original \ac{RV} a distinct auxiliary \ac{RV}. For this global construction, it is useful to introduce a notation where auxiliary \acp{RV} and auxiliary factors are indexed uniquely. Let
\[
\mathcal{I}
:=
\{(i,j) \mid \phi_j \in \bm{\Phi},\, X_i \in \mathrm{rv}(\phi_j)\}
\]
be the set of all appearances of original \acp{RV} in factors. For each incidence \((i,j) \in \mathcal{I}\), we introduce an auxiliary factor with a distinct auxiliary \ac{RV} \(\bar{X}_{i:j}\), i.e.,
\(
\bar{\phi}_{i:j}(X_i,\bar{X}_{i:j}),
\)
with $|\mathrm{Dom}(\bar{X}_{i:j})| = |\mathrm{Dom}(X_i)|$, which belongs to \ac{RV} \(X_i\) in \(\phi_j\), yielding a new set of auxiliary \acp{RV} \(\bar{\bm{X}}:=\{\bar{X}_{i:j} \mid (i,j)\in \mathcal{I}\}\). We henceforth use the following notations. For \(X_i \in \bm{X}\), let
\(
N(X_i)
:=
\{j \mid X_i \in \mathrm{rv}(\phi_j)\}
\)
be the set of factors in which \(X_i\) appears, and define
\(
\bar{\bm{X}}_i
:=
\{\bar{X}_{i:j} \mid j \in N(X_i)\}
\) and \(\bar{\bm{X}}_{\phi_j} := \{\bar{X}_{i:j} \mid X_i \in \mathrm{rv}(\phi_j) \}\).

\begin{example}
Consider a simple \ac{FG} $G = (\bm{X}, \bm{\Phi})$ with
\begin{align*}
\bm{X}= \{X_0,X_1,X_2,X_3,X_4\} ~\text{and}~ \bm{\Phi}=\{\phi_0(X_0,X_1,X_2), \phi_1(X_0,X_3,X_4)\}.
\end{align*}
Then the reparameterisation with auxiliary \acp{RV} and auxiliary factors yields
\[
\phi_0(X_0,X_1,X_2)
=
\sum_{\bar{\bm{X}}_{\phi_0}}
\bar{\phi}_0(\bar{\bm{X}}_{\phi_0})
\cdot
\bar{\phi}_{0:0}(X_0,\bar{X}_{0:0})
\bar{\phi}_{1:0}(X_1,\bar{X}_{1:0})
\bar{\phi}_{2:0}(X_2,\bar{X}_{2:0})\]
\[
\phi_1(X_0,X_3,X_4)
=
\sum_{\bar{\bm{X}}_{\phi_1}}
\bar{\phi}_1(\bar{\bm{X}}_{\phi_1})
\cdot
\bar{\phi}_{0:1}(X_0,\bar{X}_{0:1})
\bar{\phi}_{3:1}(X_3,\bar{X}_{3:1})
\bar{\phi}_{4:1}(X_4,\bar{X}_{4:1})
\] with
\[
\bar{\bm{X}_0} = \{\bar{X}_{0:0},\bar{X}_{0:1}\}, \bar{\bm{X}_1} = \{\bar{X}_{1:0}\}, \bar{\bm{X}_2} = \{\bar{X}_{2:0}\}, \bar{\bm{X}_3} = \{\bar{X}_{3:1}\}, \bar{\bm{X}_4} = \{\bar{X}_{4:1}\}.
\]
\end{example}

After applying the reparameterisation to all factors, the factorisation that yields the unnormalised joint probability distribution represented by \(G\) can be written as
\begin{align*}
\prod_{j=0}^{m-1} \phi_j(\bm{X}_j)
=
\sum_{\bar{\bm{X}}_{\phi_j}}
\prod_{j=0}^{m-1}
\bar{\phi}_j(\bar{\bm{X}}_{\phi_j})
\prod_{(i,j)\in \mathcal{I}}
\bar{\phi}_{i:j}(X_i,\bar{X}_{i:j}).
\end{align*}
Consequently, for \(\bar{\bm{X}} = \bar{\bm{X}}_{\phi_0} \cup \dots \cup \bar{\bm{X}}_{\phi_{m-1}}\), the partition function is given by
\begin{align*}
Z(G)
=
\sum_{\bm{X}}
\sum_{\bar{\bm{X}}}
\prod_{j=0}^{m-1}
\bar{\phi}_j(\bar{\bm{X}}_{\phi_j})
\prod_{(i,j)\in \mathcal{I}}
\bar{\phi}_{i:j}(X_i,\bar{X}_{i:j}).
\end{align*}
The original \acp{RV} now appear only in auxiliary factors \(\bar{\phi}_{i:j}(X_i,\bar{X}_{i:j})\). Each original \ac{RV} can be marginalised out by multiplying all corresponding auxiliary factors, yielding a new factor
\begin{align*}
\psi_i(\bar{\bm{X}}_i)
:=
\sum_{X_i}
\prod_{j\in N(X_i)}
\bar{\phi}_{i:j}(X_i,\bar{X}_{i:j}).
\end{align*}
Thus, the partition function becomes
\begin{align*}
Z(G)
=
\sum_{\bar{\bm{X}}}
\prod_{j=0}^{m-1}
\bar{\phi}_j(\bar{\bm{X}}_{\phi_j})
\prod_{i=0}^{n-1}
\psi_i(\bar{\bm{X}}_i).
\label{equation:BTNFactorisation}
\end{align*}
At this point, all original \acp{RV} have been eliminated, yielding a representation that consists only of auxiliary \acp{RV}. Moreover, each auxiliary \ac{RV} \(\bar{X}_{i:j}\) appears in at most two factors: the reparameterised factor \(\bar{\phi}_j\) and the factor \(\psi_i\) obtained by marginalising out the original \ac{RV} \(X_i\). This is precisely the structural property needed to interpret the resulting expression as a \ac{TN} where each \ac{RV} appears in at most two factors. We now make the connection to tensors explicit. Every factor over discrete \acp{RV} can be identified with a tensor by using the values of its arguments as tensor indices. Since discrete factors admit a natural tensor representation, we associate with each factor \(\phi_j \in \bm{\Phi}\) a tensor \(\mathcal{T}^{(j)}\) via a bijective mapping
\[
\tau_G : \bm{\Phi} \to \{\mathcal{T}^{(j)}\}_{j=0}^{m-1},
\quad
\tau_G(\phi_j) = \mathcal{T}^{(j)}.
\]
For \(\phi_j(\bm{X}_j)\) with \(|\bm{X}_j| = l_j\) and \(\bm{x}_j \in \mathrm{Dom}(\bm{X}_j)\) and \(\bm{d}_j = (d_{j_0},\dots,d_{j_{l_j-1}})\), we get
\[
\phi(\bm{x}_j) = \mathcal{T}^{(j)}_{\bm{x}_j},~\mathcal{T}^{(j)}\in\mathbb{R}^{\bm{d}_j}.
\]
In the context of \acp{PGM}, this establishes the well-known duality between
\acp{FG} and \acp{TN} \cite{robeva2019duality}: \acp{RV} correspond to
indices, factors correspond to tensors, and marginalisation corresponds to
contraction. We now apply this correspondence to the reparameterisation presented in \cref{lemma:FactorisationThroughARVs}.

For every incidence \((i,j)\in\mathcal{I}\), let
\(
\bm{A}^{(i:j)}
\in
\mathbb{R}^{d_i\times d_i}
\)
denote the matrix representation of the auxiliary factor.
With \(
\bar{\bm{x}}_{\phi_j}
:=
(\bar{x}_{j_0:j},\dots,\bar{x}_{j_{l_j-1}:j})
\),
the reparameterised factor \(\bar{\phi}_j\) is represented by a tensor
\(\bar{\mathcal{T}}^{(j)}\) as follows
\begin{align*}
\mathcal{T}^{(j)}_{\bm{x}_j}
=
\sum_{\bar{\bm{x}}_{\phi_j}}
\bar{\mathcal{T}}^{(j)}_{\bar{\bm{x}}_{\phi_j}}
\prod_{c=0}^{l_j-1}
\bm{A}^{(j_c:j)}_{x_{j_c},\bar{x}_{j_c:j}}.
\end{align*}
Equivalently, \(\bar{\mathcal{T}}^{(j)}\) is obtained from
\(\mathcal{T}^{(j)}\) by applying for each index the inverse matrices
\((\bm{A}^{(j_c:j)})^{-1}\) along the corresponding modes.

Next, consider the factor \(\psi_i(\bar{\bm{X}}_i)\) obtained by eliminating
an original \ac{RV} \(X_i\). Its tensor representation is denoted by
\(\mathcal{V}^{(i)}\) with entries
\begin{align*}
\mathcal{V}^{(i)}_{\bar{\bm{x}}_i}
=
\sum_{x_i\in\mathrm{Dom}(X_i)}
\prod_{j\in N(X_i)}
\bm{A}^{(i:j)}_{x_i,\bar{x}_{i:j}}.
\end{align*}
Thus, \(\mathcal{V}^{(i)}\) couples exactly the auxiliary variables
\(\bar{X}_{i:j}\) that originate from the same original \ac{RV} \(X_i\). If
the matrices \(\bm{A}^{(i:j)}\) are identity matrices they are often called copy tensors \cite{glasser2020probabilistic}. The resulting \ac{TN} is
\[
T
=
(\bm{\mathcal{T}},\bm{I}),
\quad
\bm{\mathcal{T}}
=
\{\bar{\mathcal{T}}^{(j)}\}_{j=0}^{m-1}
\cup
\{\mathcal{V}^{(i)}\}_{i=0}^{n-1},
\quad
\bm{I}
=
\{\bar{I}_{i:j}\mid (i,j)\in\mathcal{I}\}.
\]
For every incidence \((i,j)\in\mathcal{I}\), the index \(\bar{I}_{i:j}\)
connects the factor-related tensor \(\bar{\mathcal{T}}^{(j)}\) with the
variable-induced tensor \(\mathcal{V}^{(i)}\). Contracting all shared indices
yields
\begin{align*}
Z(G)
=
Z(T)
=
\sum_{\bar{\bm{x}}}
\prod_{j=0}^{m-1}
\bar{\mathcal{T}}^{(j)}_{\bar{\bm{x}}_{\phi_j}}
\prod_{i=0}^{n-1}
\mathcal{V}^{(i)}_{\bar{\bm{x}}_i}.
\label{equation:BTNContractionFromFG}
\end{align*}
We call the \ac{TN} obtained in this way a \emph{\acf{BTN}}. We define this
class of tensor networks formally as follows.

\begin{definition}[\textbf{Base Tensor Network}]
Let \(T = (\boldsymbol{\mathcal{T}},\boldsymbol{I})\) be a \ac{TN} with \(n = |\bm{I}|\) indices and \(|\bm{\mathcal{T}}| = k\) tensors. We say \(T\) is a \textbf{\acf{BTN}} if
\begin{compactenum}[(i)]
\item there exist $q \geq 1$ pairwise disjoint tensors \( \bm{\mathcal{T}}_C = \{\mathcal{B}^{(j)}\}_{j=0}^{q-1} \subset \bm{\mathcal{T}}\) which yield a tensor \(\mathcal{B} = \bigotimes_{j=0}^{q-1}\mathcal{B}^{(j)}\) such that \( \boldsymbol{I}(\mathcal{B}) = \boldsymbol{I}\), and
\item \(T\) has no free indices.
\end{compactenum}
We call \(\bm{\mathcal{T}}_C \subset \bm{\mathcal{T}}\) the \textbf{\acp{CT}} of \(T\), \(\mathcal{B}\) the \textbf{\acf{BT}} of \(T\), and the remaining tensors \(\bm{\mathcal{T}}_L = \bm{\mathcal{T}}\setminus\bm{\mathcal{T}}_C\) the \textbf{\acp{LT}} of \(T\). We denote the set of all \acp{BTN} defined over $n$ indices with dimensions \(\bm{d} = (d_0,\dots,d_{n-1})\) by \(\bm{B}_{[n,\bm{d}]}\).
\end{definition}
The definition leaves the \ac{BT} \(\mathcal{B}\) flexible. In the \ac{BTN} induced by an \ac{FG} through reparameterisation, one may choose the \ac{RV} related tensors \(\{\mathcal{V}^{(i)}\}_{i=0}^{n-1}\) as \acp{CT}, or the factor related tensors \(\{\bar{\mathcal{T}}^{(j)}\}_{j=0}^{m-1}\), since both have pairwise disjoint index sets and cover all auxiliary indices. Since \(T\) has only shared indices, every index of \(\mathcal{B}\) is connected to one \ac{LT} in \(\bm{\mathcal{T}}_L\).

\section{Tractability Through Parametric Tensor Decomposition}
\label{section:tractability_parametric_tensor_decomposition}
In this section, we present classes of \acp{BT} yielding tractability due to the presence of specific \ac{PS}. For each class, we identify the complexity of contracting the respective \acs{BTN}. In particular, we show that the complexities are only exponential in the number of indices appearing in the largest \ac{LT}. Please note that generally a wide range of PS can be considered for a \ac{BT}, certainly not all of which are covered by those presented here.
In the following, let \(T^{\mathcal{B}} \in \textbf{BTN}_{[n,\bm{d}]}\) be a \ac{BTN} with \ac{BT} \(\mathcal{B}\) and \(k\) \acp{LT}
\(
\bm{\mathcal{T}}_L
=
\bm{\mathcal{T}}\setminus \bm{\mathcal{T}}_C
=
\{\mathcal{L}^{(0)},\dots,\mathcal{L}^{(k-1)}\}.
\)
For each \ac{LT} \(\mathcal{L}^{(j)}\), define
\(
\bm{I}_j := \bm{I}(\mathcal{L}^{(j)}),
n_j := |\bm{I}_j|, \text{and let}~
n_{\max} := \max_{0\leq j<k} n_j.
\)
We now identify several classes of \acp{BT} for which contraction becomes tractable due to the presence of \ac{PS}. While the \ac{BTN} framework is defined for arbitrary index dimensions, we focus on the uniform-dimensional setting \(\bm{d}=(d,\dots,d)\) in this work. For the respective results, we assume the structures to already be present. Proofs of the following results are provided in the appendix.

\begin{lemma}[CP Base Tensor]
\label[lemma]{lemma:CPBT}
Let \(T^{\mathcal{B}}\in \bm{B}_{[n,\bm{d}]}\) be a \ac{BTN} with \(\mathcal{B}\) admitting a CP decomposition of rank \(r\), i.e.,
\begin{equation*}
\label{equation:CPBT}
\mathcal{B}
=
\sum_{i=0}^{r-1}
\bm{v}_0^{(i)}\otimes \cdots \otimes \bm{v}_{n-1}^{(i)}.
\end{equation*}
Then \(Z(T^{\mathcal{B}})\) can be computed with complexity
\(
\mathcal{O}\!\left(rkd^{n_{\max}}\right).
\)
In particular, if \(\mathcal{B}\) is rank one, then
\(Z(T^{\mathcal{B}})\) can be computed with complexity
\(
\mathcal{O}\!\left(kd^{n_{\max}}\right).
\)
\end{lemma}

The CP case shows that low-rank \acp{BT} decouple the global index dependence.
Each \ac{LT} is contracted locally against vectors, and no intermediate
tensor larger than the largest \ac{LT} has to be formed. Next, we consider symmetric \acp{BT} for which we can use a compact count-space representations.

\begin{lemma}[Symmetric Base Tensor]
\label[lemma]{lemma:SymmetricBT}
Let \(T^{\mathcal{B}}\in \bm{B}_{[n,\bm{d}]}\) be a \ac{BTN} with \(\mathcal{B}\) being fully symmetric. Define
\[
c_1 := \binom{n_{\max}+d-1}{d-1},
\qquad
c_2 := \binom{n-n_{\max}+d-1}{d-1}.
\]
Then \(Z(T^{\mathcal{B}})\) can be computed with complexity
\(
\mathcal{O}\!\left(k\left(d^{n_{\max}}+c_1c_2\right)\right).
\)
\end{lemma}
Symmetric \acp{BT} benefit from a compact count-space representation that requires for an \(n\)-th order tensor only \(\binom{n+d-1}{d-1}\) elements to be specified rather than \(d^n\) \cite{kolda2009tensor}. Moreover, any contraction with an \ac{LT} yields again a symmetric \ac{BT}, meaning that symmetries and thus compact encodings are retained during contraction. We next show that the previous two results can also be combined. 

\begin{definition}[Symmetry-Rank-One Tensor]
\label{definition:SymmetricRankOneTensor}
Let \(\mathcal{S}\in \mathbb{R}^{\bm{d}}\) be symmetric and
\(
\mathcal{R} \in \mathbb{R}^{\bm{d}}
\)
be rank one. We call
\(
\mathcal{Q}
=
\mathcal{S}\odot \mathcal{R}
\)
a symmetry-rank-one tensor and the decomposition
\[
\mathcal{B}
=
\sum_{i=0}^{r-1}\alpha_i\mathcal{Q}^{(i)}
=
\sum_{i=0}^{r-1}\alpha_i
\left(
\mathcal{S}^{(i)}\odot
\mathcal{R}^{(i)}
\right)
\]
a \emph{symmetry-CP decomposition} with \(\alpha_0,\dots,\alpha_{r-1} \in \mathbb{R}\).
\end{definition}

\begin{lemma}[Symmetry-CP Base Tensor]
\label[lemma]{lemma:SymmetricRankOneBT}
Let \(T^{\mathcal{B}}\in \bm{B}_{[n,\bm{d}]}\) be a \ac{BTN} where
\(\mathcal{B}\) admits a symmetry-CP decomposition of rank \(r\).
Then \(Z(T^{\mathcal{B}})\) can be computed with complexity
\(
\mathcal{O}\!\left(rk\left(d^{n_{\max}}+c_1c_2\right)\right),
\)
where \(c_1\) and \(c_2\) are as in \cref{lemma:SymmetricBT}.
\end{lemma}

Symmetry-CP decompositions combine the advantages of rank-one and symmetric
structure. The rank-one part separates the dependence on individual indices and
can be absorbed locally, while symmetries are captured through compact count-space representations.

The previous results show that \acp{BTN} are tractable if their \ac{BT} belongs
to a suitable parametric class. However, after transforming a general \ac{FG}
into a \ac{BTN}, the resulting \ac{BT} will usually not exactly have such a
structure. The central idea is therefore to decompose the \ac{BT} additively
into tractable components.

\begin{lemma}[Additive Decomposition]
\label[lemma]{lemma:AdditiveBTDecomposition}
Let \(T^{\mathcal{B}}\in\bm{B}_{[n,\bm{d}]}\) be a \ac{BTN} with fixed
\acp{LT} \(\bm{\mathcal{T}}_L\) and with \ac{BT}
\(
\mathcal{B}
=
\sum_{i=0}^{r-1}\alpha_i\mathcal{B}^{(i)}.
\)
Let \(T^{\mathcal{B}^{(i)}}\) denote the \ac{BTN} obtained by replacing
\(\mathcal{B}\) with \(\mathcal{B}^{(i)}\), while keeping \(\bm{\mathcal{T}}_L\)
fixed. Then
\(
Z(T^{\mathcal{B}})
=
\sum_{i=0}^{r-1}\alpha_i Z(T^{\mathcal{B}^{(i)}}).
\)
\end{lemma}

\begin{proof}
With fixed \acp{LT} \(\bm{\mathcal{T}}_L\), the contraction is linear in the
entries of the \ac{BT}. Writing the contraction explicitly gives
\[
Z(T^{\mathcal{B}})
=
\sum_{\bm{x}}
\mathcal{B}_{\bm{x}}
\prod_{j=0}^{k-1}
\mathcal{L}^{(j)}_{\bm{x}_{\bm{I}_j}} .
\]
Substituting
\(\mathcal{B}=\sum_{i=0}^{r-1}\alpha_i\mathcal{B}^{(i)}\) yields
\[
\begin{aligned}
Z(T^{\mathcal{B}})
&=
\sum_{\bm{x}}
\left(
\sum_{i=0}^{r-1}
\alpha_i\mathcal{B}^{(i)}_{\bm{x}}
\right)
\prod_{j=0}^{k-1}
\mathcal{L}^{(j)}_{\bm{x}_{\bm{I}_j}}
\\
&=
\sum_{i=0}^{r-1}
\alpha_i
\sum_{\bm{x}}
\mathcal{B}^{(i)}_{\bm{x}}
\prod_{j=0}^{k-1}
\mathcal{L}^{(j)}_{\bm{x}_{\bm{I}_j}}
=
\sum_{i=0}^{r-1}
\alpha_i Z(T^{\mathcal{B}^{(i)}}).
\end{aligned}
\]
\end{proof}

By \cref{lemma:AdditiveBTDecomposition}, decomposing the \ac{BT} into multiple \acp{BT} introduces a set of \acp{BTN} whose contraction results yield the original contraction result. This motivates learning an approximate \ac{BT} \(\widetilde{\mathcal{B}} \approx \mathcal{B}\) that permits a tractable decomposition such that
\[
Z(T^{\mathcal{B}})
\approx
Z(T^{\widetilde{\mathcal{B}}})
=
\sum_{i=0}^{r-1}\alpha_i Z(T^{\widetilde{\mathcal{B}}^{(i)}}).
\]
Thus, starting from a given \ac{FG}, inference in the corresponding \ac{BTN} is reduced to two tasks. First, identify \acp{BT} for which \(Z(T^{\mathcal{B}})\) is tractable due to \ac{PS} exploitation. Second,
decompose or approximate a given \ac{BT} by such tractable components. In the next section, we continue with approximating the initial \ac{BT} \(\mathcal{B}\) with a \ac{BT} \(\widetilde{\mathcal{B}}\) that permits a decomposition into tractable components.

\section{Base Tensor Frobenius Norm Minimisation}
\label{section:learning_base_tensor_networks}

\Acp{BTN} can be used for both exact and approximate inference. Exact inference is possible when the induced \ac{BT} already has tractable \ac{PS}, or when it admits an exact additive decomposition into tractable \acp{BT}. This is particularly interesting when considering learned high-treewidth \acp{FG} starting from tractable \acp{BTN}. In this work, however, we focus on approximate inference. Given a \ac{BTN} \(T^{\mathcal{B}}\in\bm{B}_{[n,\bm{d}]}\) obtained from an \ac{FG} \(G\), we consider an approximation \(\mathcal{B} \approx \widetilde{\mathcal{B}}\) in terms of an additive \ac{PTD} such that
\begin{equation*}
\label{equation:LearnedPTDBT}
\widetilde{\mathcal{B}}
=
\sum_{i=0}^{r-1}\alpha_i\widetilde{\mathcal{B}}^{(i)},
\end{equation*}
where the components \(\widetilde{\mathcal{B}}^{(i)}\) are chosen from tractable classes. Thus, instead of requiring useful \ac{PS} to be present in the original \ac{FG}, we explicitly build and use it in the induced \ac{BT}.

The objective we consider in this work is the Frobenius norm minimisation, since we will shortly see that it gives a useful objective whose terms can be effectively evaluated using the topology and \ac{PS} of a \ac{BTN}. While we omit error bounds in this work, we believe the Frobenius norm minimisation provides a useful foundation for developing more sophisticated learning approaches for \acp{BTN}. There are several possible learning strategies, where generally a larger decomposition is able to yield more accurate approximations. One may only optimise the weights \(\bm{\alpha}\) of a fixed \ac{PTD}, jointly learn rank-one and symmetric parameters, or also learn suitable invertible maps used in the reparameterisation from the original \ac{FG} to the \ac{BTN}. We focus in the following on learning the weights with symmetry-rank-one tensors \(\widetilde{\mathcal{B}}^{(i)}=\mathcal{S}^{(i)}\odot\mathcal{R}^{(i)}\), with \(\mathcal{S}^{(i)}\) symmetric and \(\mathcal{R}^{(i)}\) rank one. The Frobenius norm learning objective is
\begin{equation*}
\label{equation:BTNLearningObjective}
\min
\left\|
\mathcal{B}
-
\widetilde{\mathcal{B}}
\right\|_F^2
=
\min
\left\|
\mathcal{B}
-
\sum_{i=0}^{r-1}\alpha_i\widetilde{\mathcal{B}}^{(i)}
\right\|_F^2.
\end{equation*}

The minimisation can be performed effectively here as the norm expands to
\begin{align*}
\label{equation:BTNLearningObjectiveExpanded}
\left\|
\mathcal{B}
-
\sum_{i=0}^{r-1}\alpha_i\widetilde{\mathcal{B}}^{(i)}
\right\|_F^2
&=
\left\|\mathcal{B}\right\|_F^2
+
\sum_{i_1,i_2=0}^{r-1}\alpha_{i_1}\alpha_{i_2}
\left\langle\widetilde{\mathcal{B}}^{(i_1)},\widetilde{\mathcal{B}}^{(i_2)}\right\rangle_F
-2\sum_{i=0}^{r-1}\alpha_i
\left\langle\mathcal{B},\widetilde{\mathcal{B}}^{(i)}\right\rangle_F .
\end{align*}
Thus, the objective only requires Frobenius norms and Frobenius scalar products. These quantities can be evaluated efficiently for the tensors considered here without materialisation. First, the \ac{BT} is an outer product of disjoint \acp{CT}, yielding
\begin{equation*}
\label{equation:BTAsCoreTensorProduct}
\mathcal{B}=\bigotimes_{c=0}^{q-1}\mathcal{B}^{(c)},
\qquad
\left\|\mathcal{B}\right\|_F^2
=
\prod_{c=0}^{q-1}\left\|\mathcal{B}^{(c)}\right\|_F^2.
\end{equation*}
Hence, the norm of the full \ac{BT} does not require materialising the exponentially large tensor. Second, the component scalar products form the Gram matrix
\begin{equation*}
\label{equation:BTNLearningGramMatrix}
\bm{G}_{ij}
=
\left\langle
\widetilde{\mathcal{B}}^{(i)},
\widetilde{\mathcal{B}}^{(j)}
\right\rangle_F .
\end{equation*}
For rank-one tensors, such scalar products factorise into vector scalar products. For symmetric tensors, they can be computed in compact count spaces. Consequently, they can be evaluated for symmetry-rank-one components without materialising the full tensors, combining both advantages. Third, the terms
\begin{equation*}
\label{equation:BTNLearningTargetVector}
\bm{b}_i
=
\left\langle
\mathcal{B},
\widetilde{\mathcal{B}}^{(i)}
\right\rangle_F
\end{equation*}
can be evaluated by another contraction exploiting that the \ac{BT} is an outer product and that the \(\widetilde{\mathcal{B}}^{(i)}\) are tractable components with \ac{PS}. If the components are fixed, learning the weights \(\bm{\alpha}\) reduces to the optimisation problem
\begin{equation*}
\label{equation:BTNLearningGramObjective}
\min_{\bm{\alpha}}
\left(
\left\|\mathcal{B}\right\|_F^2
+
\bm{\alpha}^{\top}\bm{G}\bm{\alpha}
-2\bm{\alpha}^{\top}\bm{b}
\right),
\end{equation*}
possibly with constraints such as \(\bm{\alpha}\geq 0\). More expressive variants may also optimise the component parameters or the invertible reparameterisation maps, but this yields larger and typically non-convex learning problems which require more sophisticated approaches.

Once \(\widetilde{\mathcal{B}}\) has been learned, inference follows from \cref{lemma:AdditiveBTDecomposition}
where each \(T^{\widetilde{\mathcal{B}}^{(i)}}\) keeps the original \acp{LT} fixed and replaces the \ac{BT} by a tractable component. Thus, a learned \ac{PTD} turns approximate inference into a weighted sum of tractable \ac{BTN} contractions, where both learning and contraction exploit \ac{PS}. Most importantly, tractability is determined by the \ac{PS} of the respective \acp{BT}, rather than solely by \ac{RV} (in-)dependencies within the original \ac{FG}. Hence, \acp{BTN} provide an effective representation for probabilistic inference beyond classical treewidth notions of tractability.

\section{Related Work}
\label{section:related_work}

Due to the natural correspondence between joint distributions and higher-order tensors, research at the intersection of \acp{PGM} and \acp{TN} has received considerable attention \cite{robeva2019duality,glasser2020probabilistic,glasser2019expressive,dudek2019efficient,miller2021probabilistic,ahn2018bucket}. Transformations of \acp{PGM} into \acp{TN} are commonly used and often rely on copy tensors \cite{glasser2020probabilistic,glasser2019expressive,pancotti2023one,miller2021probabilistic,dudek2019efficient,gray2021hyper}. Tensor decompositions have also been studied for inference, learning, and representation in \acp{PGM} \cite{novikov2014putting,bonnevie2021matrix,rabanser2017introduction,savicky2007exploiting,hsiao2022fast,ducamp2020efficient,vomlel2014approximate}. These approaches exploit compact representations induced by \ac{PS}. However, their effectiveness usually depends on such structure being explicitly present or determined through a specific representation.

Probabilistic circuits form a closely related line of work. They represent probability distributions by computational graphs whose tractability follows from structural restrictions such as decomposability, smoothness, and determinism \cite{chavira2008probabilistic,choi2020probabilistic}. This follows a similar decomposition idea where large probabilistic computations are factorised into smaller subcomputations that can be combined efficiently. In contrast, our approach keeps the \ac{BTN} topology fixed and seeks tractability through a \ac{PTD} of the \ac{BT}. Thus, probabilistic circuits expose tractability through the computational graph, whereas \acp{BTN} expose it through \ac{PS} in the \ac{BT}. Decomposition-based inference is also central to methods such as cutset conditioning and recursive conditioning \cite{darwiche2001recursive}. These methods mainly trade time for space by simplifying the graphical topology.

\section{Conclusion}
\label{section:conclusion}

In this work, we introduce \acp{BTN} as a novel framework for probabilistic inference in which tractability is entirely controlled by the \ac{PS} of a single tensor, the \ac{BT}. Starting from a general \ac{FG}, we use auxiliary \acp{RV} and invertible maps to obtain a \ac{BTN} whose contraction equals the partition function of the original model. Crucially, the \ac{BT} is the central object for tractability analysis. We identify tractable classes of \acp{BT}, including rank-one, symmetric, and symmetry-rank-one tensors, and introduce \ac{PTD} as an additive decomposition principle. If the \ac{BT} is decomposed into tractable components, the contraction of the original \ac{BTN} can be computed as a weighted sum of tractable contractions. For general \acp{PGM}, the induced \ac{BT} will usually not be exactly tractable. We therefore formulated approximate inference as learning a structured approximation of the \ac{BT}. Frobenius norm minimisation provides a general objective whose terms can be evaluated using the product structure of the \ac{BT} and the \ac{PS} of the components. Once the approximation is learned, the additive decomposition lemma directly yields an approximate inference procedure for the partition function. Overall, \acp{BTN} provide a way to exploit latent \ac{PS} beyond low-treewidth settings or explicitly structured models. Future work includes deriving error bounds between tensor approximation and inference error, jointly learning decomposition parameters and invertible maps, extending the framework to further tractable tensor classes, and empirically evaluating the approach on larger \acp{PGM}.

\begin{credits}
\subsubsection{\discintname}
The authors have no competing interests to declare that are
relevant to the content of this article.
\end{credits}
\bibliographystyle{splncs04}
\bibliography{main_references}

\newpage
\appendix
\begin{center}
\textbf{Appendix}    
\end{center}

\section{Proofs}

\begin{proof}[\textbf{Proof of Lemma 1}]
We first show that a single \ac{RV} in a factor can be replaced by an auxiliary \ac{RV} through an auxiliary factor. Applying this construction successively to all \acp{RV} in the factor then yields the claim. Introduce for \(X_{i_j}\in \bm{X_i}\) an auxiliary \ac{RV} \(\bar X_{i_j}\) with \(|\mathrm{Dom}(\bar X_{ij})|=|\mathrm{Dom}(X_{i_j})|=d_{i_j}\). Let

\[
\delta_{ij}(X_{i_j},X'_{i_j})
\]
be the equality factor defined over the original and the auxiliary \ac{RV} that assigns $1$ if all \ac{RV} assignments are equal and otherwise $0$. This equality factor corresponds to an identity matrix \(\mathbf I_{d_{i_j}}\) with ones on the diagonal. Since the auxiliary factor \(\bar\phi_{i_j}(X_{i_j},\bar X_{i_j})\) corresponds to an invertible matrix \(\mathbf A^{(i_j)}\), we can write
\[
\mathbf I_{d_{i_j}}
=
\mathbf A^{(i_j)}
\bigl(\mathbf A^{(i_j)}\bigr)^{-1}.
\]
With \(\bar{\phi}^{-1}_{i_j}(\bar X_{i_j},X'_{i_j})\) denoting the factor that corresponds to \((\mathbf A^{(i_j)})^{-1}\), we get
\[
\delta_{i_j}(X_{i_j},X'_{i_j})
=
\sum_{\bar X_{i_j}}
\bar\phi_{i_j}(X_{ij},\bar X_{ij})
\phi^{-1}_{i_j}(\bar X_{ij},X'_{ij}),
\]
since factor multiplication with a subsequent marginalisation coincides with a matrix multiplication.
Multiplying \(\phi^{-1}_{i_j}\) into \(\phi_i\) and marginalising \(X'_{i_j}\) therefore replaces \(X_{i_j}\) in \(\phi_i\) by \(\bar{X}_{i_j}\). Let
\[
\phi_i^{(j)}(\mathbf X_i\setminus\{X_{i_j}\},\bar X_{i_j})
=
\sum_{X'_{i_j}}
\phi_i(\mathbf X_i\setminus\{X_{i_j}\},X'_{i_j})
\phi^{-1}_{i_j}(\bar X_{i_j},X'_{i_j}).
\]
Then
\[
\begin{aligned}
\sum_{\bar X_{i_j}}
\phi_i^{(j)}(\mathbf X_i\setminus\{X_{i_j}\},\bar X_{i_j})
\bar\phi_{i_j}(X_{i_j},\bar X_{i_j})
&=
\sum_{X'_{i_j}}
\phi_i(\mathbf X_i\setminus\{X_{i_j}\},X'_{i_j})
\delta_{i_j}(X_{i_j},X'_{i_j})\\
&=
\phi_i(\mathbf X_i).
\end{aligned}
\]
Repeating this reparameterisation successively for every \(X_{i_j}\in\mathbf X_i\) finally yields a factor \(\bar\phi_i(\bar{\mathbf X}_i)\) satisfying
\[
\phi_i(\mathbf X_i)
=
\sum_{\bar{\mathbf X}_i}
\bar\phi_i(\bar{\mathbf X}_i)
\prod_{j=0}^{l-1}
\bar\phi_{ij}(X_{ij},\bar X_{ij}).
\]
\end{proof}

\begin{proof}[\textbf{Proof of Lemma 2}]
Using the CP decomposition of \(\mathcal{B}\), we obtain
\[
\begin{aligned}
Z(T^{\mathcal{B}})
&=
\sum_{i=0}^{r-1}
\sum_{\bm{x}}
\left(
\bigotimes_{l=0}^{n-1} v_l^{(i)}
\right)
\prod_{j=0}^{k-1}
\mathcal{L}^{(j)}_{\bm{x}_{\bm{I}_j}}
\\
&=
\sum_{i=0}^{r-1}
\prod_{j=0}^{k-1}
\left(
\sum_{\bm{x}_{\bm{I}_j}}
\mathcal{L}^{(j)}_{\bm{x}_{\bm{I}_j}}
\bigotimes_{l\in \bm{I}_j}
v_l^{(i)}
\right).
\end{aligned}
\]
For a single \ac{LT}, absorbing the rank-one terms and subsequently contracting its indices can be done in \(\mathcal{O}(d^{n_{\max}})\). Computing all contractions for all \(r\) CP components and all \(k\) \acp{LT} therefore costs
\[
\mathcal{O}\!\left(rkd^{n_{\max}}\right).
\]
For \(r=1\), this reduces to
\(\mathcal{O}\!\left(kd^{n_{\max}}\right)\).
\end{proof}

\begin{proof}[\textbf{Proof of Lemma 3}]
Let \(\mathrm{count}(\bm{x}) = (c_0,\dots,c_{d-1})\) with \(c_i := |\{j\mid x_j = i\}|\) and 
\(\mathcal C_n:=\{\mathrm{count}(\bm{x})\mid\bm{x} \in \mathrm{Dom}(\bm{X})\}\)
denote the set of count vectors over \(n\) \acp{RV} (indices) where $c_i$ counts the number of \acp{RV} (indices) having assignment $i$. By symmetry, an \(n\)-th order symmetric tensor is fully determined by one entry for each \(\bm{c}\in\mathcal C_n\), where
\[
|\mathcal C_n|=\binom{n+d-1}{d-1},
\]
due to the fact that assignments being permutations of each other are mapped to the same value.
Now consider contracting a symmetric tensor \(S\) of order \(n\) with an LT \(L^{(j)}\) over \(n_j\) indices. For \(\bm{c}_j \in\mathcal C_{n_j}\), first aggregate the entries of \(L^{(j)}\) according to their count vectors (i.e., aggregate all permutations)
Constructing this representation requires iterating over the \(d^{n_j}\) entries of \(L^{(j)}\) and hence has complexity \(O(d^{n_j})\).
For each count vector \(\bm{c}'\in\mathcal C_{n-n_j}\) of the remaining indices, the contracted tensor is then given in count space by
\[
S'_{\bm{c}'}
=
\sum_{\bm{c}_j\in\mathcal C_{n_j}}
S_{\bm{c}_j+\bm{c}'}\,
\widehat L^{(j)}_{\bm{c}_j}.
\]
Thus, the contraction requires
\[
O\left(
|\mathcal C_{n_j}|
|\mathcal C_{n-n_j}|
\right)
=
O\left(
\binom{n_j+d-1}{d-1}
\binom{n-n_j+d-1}{d-1}
\right)
\]
operations. Moreover, \(S'\) depends only on the count vector \(\bm{c}'\) and is therefore again symmetric.
Now contract first an LT with \(n_{\max}\) indices. For this contraction,
\[
|\mathcal C_{n_{\max}}|=c_1,
\qquad
|\mathcal C_{n-n_{\max}}|=c_2,
\]
and its complexity is
\[
O\left(d^{n_{\max}}+c_1c_2\right).
\]
Afterwards, the remaining symmetric tensor has order at most \(n-n_{\max}\). Hence, each of the remaining LT contractions has complexity at most
\[
O\left(d^{n_{\max}}+c_1c_2\right).
\]
Contracting all LTs therefore yields
\[
O\left(k\left(d^{n_{\max}}+c_1c_2\right)\right).
\]
\end{proof}

\begin{proof}[\textbf{Proof of Lemma 4}]
For a single symmetry-rank-one component
\[
\mathcal{Q}^{(i)}
=
\mathcal{S}^{(i)}
\odot
\left(
\bm{v}_0^{(i)}\otimes\cdots\otimes \bm{v}_{n-1}^{(i)}
\right),
\]
we first absorb the rank-one vectors into the \acp{LT}.
Then contracting \(T^{\mathcal{Q}^{(i)}}\) is equivalent to contracting a
\ac{BTN} whose \ac{BT} is the symmetric tensor \(\mathcal{S}^{(i)}\) and whose
\acp{LT} are the updated \acp{LT} after absorbing the rank-one terms. By
\cref{lemma:SymmetricBT}, this costs
\[
\mathcal{O}\!\left(k\left(d^{n_{\max}}+c_1c_2\right)\right).
\]
Using the additive decomposition
\[
\mathcal{B}
=
\sum_{i=0}^{r-1}
\alpha_i\mathcal{Q}^{(i)},
\]
and linearity of contraction in the \ac{BT}, the full contraction is the
weighted sum of the \(r\) component contractions. Hence the total complexity is
\[
\mathcal{O}\!\left(rk\left(d^{n_{\max}}+c_1c_2\right)\right).
\]
\end{proof}
\end{document}